\documentclass[letterpaper, 10 pt, conference]{ieeeconf}
\IEEEoverridecommandlockouts
\usepackage{cite}
\usepackage[bookmarks=true]{hyperref}
\usepackage{color}
\usepackage{times}
\usepackage{epsfig}
\usepackage{epstopdf}
\usepackage{amsmath}
\usepackage{amssymb}
\usepackage{bbm}
\usepackage[ruled,vlined,linesnumbered]{algorithm2e}
\usepackage{graphicx}
\usepackage{verbatim}
\usepackage{booktabs}
\usepackage{multirow}
\usepackage{makecell}
\usepackage{subfigure}
\usepackage{lscape}
\usepackage{url}
\usepackage{makecell}
\usepackage{xcolor,colortbl}
\usepackage{float}

\newcommand{\XX}{\mathbf{X}}

\newcommand{\TT}{\mathbf{T}}
\newcommand{\ttt}{\mathbf{t}}
\newcommand{\OO}{\mathbf{O}}
\newcommand{\ooo}{\mathbf{o}}
\newcommand{\YY}{\mathbf{Y}}
\newcommand{\yyy}{\mathbf{y}}
\newcommand{\WW}{\mathbf{W}}
\newcommand{\www}{\mathbf{w}}
\newcommand{\UU}{\mathbf{U}}
\newcommand{\uuu}{\mathbf{u}}
\newcommand{\R}{\mathbb{R}}

\title{\LARGE \bf Simultaneous Learning from Human Pose and Object Cues for Real-Time Activity Recognition}

\author{
    Brian Reily$^{1}$, Qingzhao Zhu$^{1}$, Christopher Reardon$^{2}$, and Hao Zhang$^{1}$%
    \thanks{*This work was partially supported by NSF IIS-1942056, ARL DCIST CRA W911NF-17-2-0181, USAFA FA7000-18-2-0016, and ARO W911NF-17-1-0447.}%
    \thanks{$^{1}$Brian Reily, Qingzhao Zhu and Hao Zhang are with Human-Centered Robotics Laboratory
    at the Colorado School of Mines, Golden, CO, 80401, USA.
    Email: \{breily, zhuqingzhao, hzhang\}@mines.edu.}%
    \thanks{$^{2}$Christopher Reardon is with U.S. Army Research Laboratory, Adelphi, MD, 20783, USA.
    Email: christopher.m.reardon3.civ@mail.mil.}%
}

\begin{document}

\newtheorem{definition}{Definition}
\newtheorem{theorem}{\textbf{Theorem}}
\newtheorem{lemma}{\textbf{Lemma}}
\newtheorem{proposition}{Proposition}
\newtheorem{property}{Property}
\newtheorem{observation}{Observation}
\newtheorem{corollary}{\textbf{Corollary}}

\maketitle
\thispagestyle{empty}
\pagestyle{empty}

\begin{abstract}

Real-time human activity recognition plays an essential role 
in real-world human-centered robotics applications,
such as assisted living and human-robot collaboration.
Although previous methods based on skeletal data to encode human poses
showed promising results on real-time activity recognition,
they lacked the capability to consider the context provided by objects within the scene and in use by the humans,
which can provide a further discriminant between human activity categories.
In this paper, we propose a novel approach to 
real-time human activity recognition,
through simultaneously learning from observations
of both human poses and objects involved in the human activity.
We formulate human activity recognition as a joint optimization problem under a unified mathematical framework,
which uses a regression-like loss function to integrate human pose and object cues
and defines structured sparsity-inducing norms to identify discriminative body joints and object attributes.
To evaluate our method,
we perform extensive experiments on two benchmark datasets
and a physical robot in a home assistance setting.
Experimental results have shown that
our method outperforms previous methods
and obtains real-time performance for human activity recognition with a processing speed of $10^4$ Hz.


\end{abstract}

\section{Introduction}
\label{sec:intro}




Real-time human activity recognition is an essential capability of robots
in human-centered robotics applications,
such as assisted living, service robotics, human-robot teaming, and human-robot interaction
\cite{kruijff2014experience,fong2017human,schulz2018preferred,zhang2015adaptive}.
It allows intelligent robots to understand human behaviors in a timely manner
in order to effectively assist and interact with humans.
Human activity recognition by robots in the real world is a difficult problem,
complicated by both variations in human appearances and poses, 
and by challenges such as illumination changes or occlusions.
Given these challenges, 
it is important for a robot to extract as much relevant information as possible from 
sensory observations.
For example, as illustrated in Figure \ref{fig:motivation},
this information can consist of humans themselves,
such as human poses encoded by the human skeleton representation,
and the context from the objects in the environment
and objects that the human is interacting with,
which provide additional cues to recognize activities.
Moreover,
in most real-world robotics applications,
and especially in time-critical scenarios,
activity recognition must occur in real-time
so that a robot can promptly interact with and assist humans.

Due to the importance of activity recognition,
many methods have been introduced over the past few decades \cite{han2017space,vrigkas2015review,aggarwal2011human,aggarwal2014human}.
Especially, techniques based on skeletal data from structured-light cameras \cite{shotton2011real}
have attracted  increasing attention,
due to skeletal data's real-time performance
and invariance to viewing distances and angles.
For example,
the methods can be implemented based on hand-crafted skeletal features
\cite{reily2018skeleton},
a concatenation of multiple types of features \cite{masood2011measuring,wang2014learning,yu2014discriminative},
and learning skeleton-based representations, e.g.,
by sparse optimization \cite{han2017simultaneous,fanello2013keep} or deep learning \cite{wang2018temporal,wang2019deep}.
These methods generally have the limitation of not learning from the
context of objects that the human is interacting with.
Although several methods used object information \cite{wei2013modeling,koppula2013learning},
they require explicit knowledge of the objects,
such as object affordances that are typically manually defined to describe how an object can be interacted with. 
Moreover, previous methods cannot estimate the importance of the objects 
in recognizing human activities.

\begin{figure}[t]
    \centering
    \vspace{5pt}
    \includegraphics[width=0.475\textwidth]{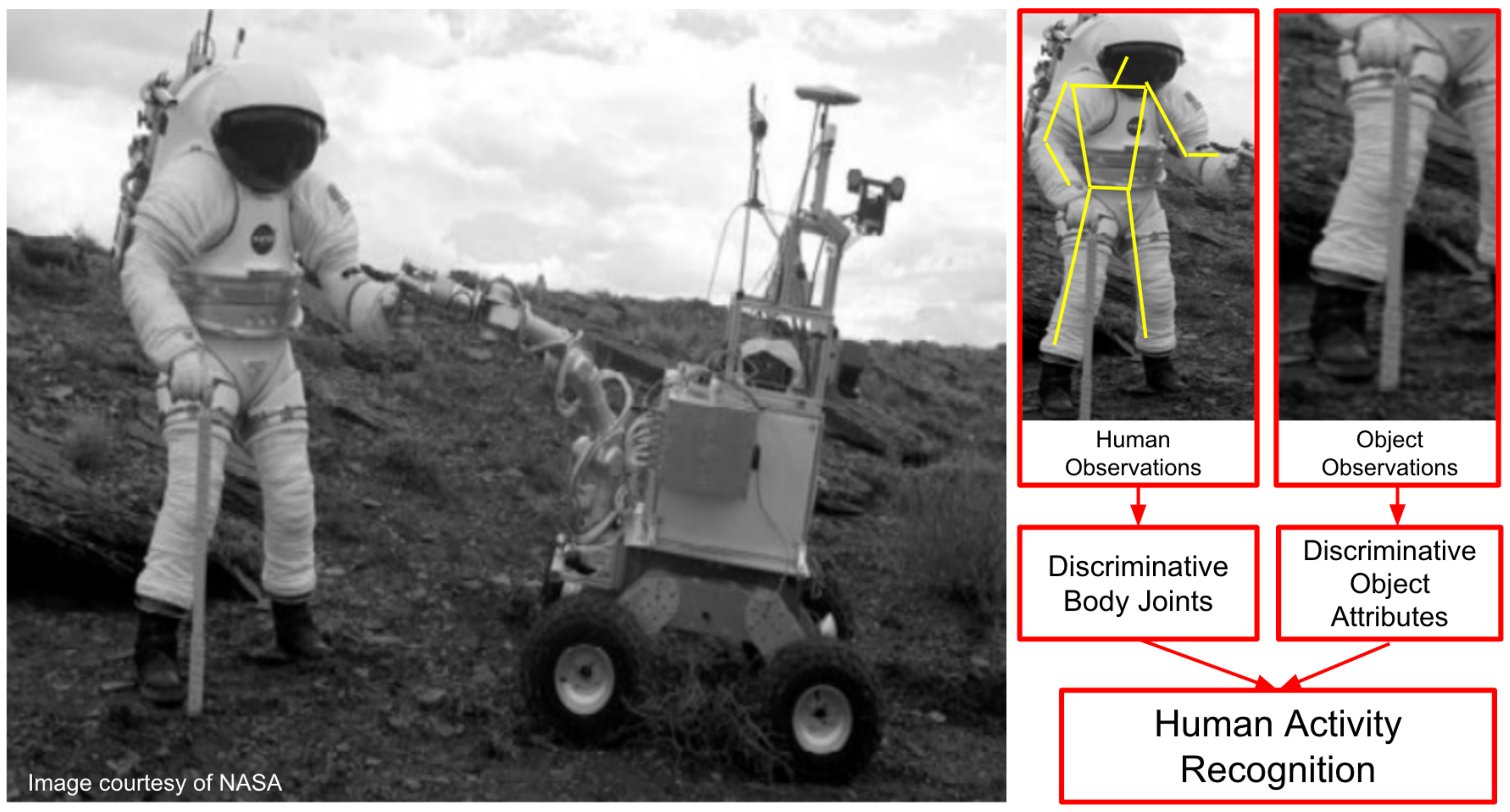}
    \vspace{-3pt}
    \caption{
    A motivating example of integrating observations of human poses and of the objects involved in the task to understand human activities.
    The objects within the scene and in use by the humans provide additional cues beside human poses to recognize activities.
    }
    \label{fig:motivation}
    \vspace{-6pt}
\end{figure}

In this paper, we propose a principled method for real-time human activity recognition
based on learning simultaneously from observations of humans and objects.
Our approach formulates activity recognition as a regression-like optimization problem,
and applies structured norms as regularization terms to promote sparsity and identify
discriminative skeletal joints and object attributes.
This formulation is inspired by the fact that many activities rely solely
on a subset of joints (e.g., a waving activity uses only joints in the arm, not in the legs),
or can be recognized
based on context of objects in the scene 
(e.g., reading a book and typing on a laptop
at a table appear similar if only the human pose is considered).
By learning the importance of both to human activities,
the proposed method is able to identify and integrate more relevant information
to improve activity recognition accuracy.
Because classification is integrated in our regression-like convex objective function
(i.e., no separate classifier is needed),
our approach is capable of operating in real-time,
which makes it suitable for robotics applications with real-time needs.

This paper has two major contributions:
\begin{enumerate}
\item We propose a novel principled method
that formulates human activity recognition
as simultaneously learning from human and object observations
based on a unified regularized optimization framework.
The method identifies both discriminative skeletal joints
and discriminative object attributes,
and integrates classification with sparse representation learning in order
to enable a high processing speed for real-time recognition.

\item We implement a new iterative optimization algorithm to
solve the formulated regularized optimization problem that has dependent model parameters,
which holds a theoretical guarantee to find the optimal solution.
\end{enumerate}


\section{Related Work}\label{sec:related}



Activity recognition has been shown to be a critical ability for 
robots to work with people in real-world human-centered robotics applications
\cite{rossi2017user,stone2010ad,albrecht2018autonomous,heard2018diagnostic,al2017learning,zhang2015bio}.
This section provides a brief review of existing methods of skeleton-based representations 
and object-assisted activity recognition.



\subsection{Skeleton-Based Representations for Activity Recognition}

Among diverse human representations, skeleton-based representations attracted
extensive attentions since the availability of structured-light or color-depth cameras.
Skeleton-based representations can be based on joint positions, displacement,
orientation, 
and a combination of multiple joint features \cite{han2017space}.

Relative \emph{spatial displacement} between a pair of body joints
is one of the most commonly applied skeleton-based features.
For example, the normalized joint positions were employed to compute pairwise relative distances between joints as features to categorize activities using extreme learning machines \cite{chen2013online}.
Euclidian distances betweens joint pairs were computed as skeletal features
to recognize activities \cite{wang2014learning}.
Rahmani \textit{et al.} \cite{rahmani2014real} chose a reference joint,
such as the torso center,
and computed relative distances to other body joints.
In addition, 
\emph{orientation} of a segment between two joints in space or time
is widely studied.
Boubou and Suzuki \cite{boubou2015classifying} implemented the 
histogram of oriented velocity vector features
that calculate joint velocities between frames and use distributions of joint orientations to classify human activities.
Yang and Tian \cite{yang2014effective} designed 
a descriptor to include joint orientation differences between frames as features.
\emph{Joint positions} were also directly applied 
as input into long short-term memory networks and recurrent neural networks to recognize activities \cite{zhu2016co}.
Recent methods \emph{integrated multiple features}.
Luo \textit{et al.}\cite{luo2014spatio} fused sparse-coding skeleton features with a representation named center-symmetric motion local ternary pattern,
which extracts spatial and temporal gradients as features.
A learning-based method is proposed in \cite{han2017simultaneous}
to estimate the weights of feature modalities for activity recognition.

While most previous methods only consider skeletal information to build representations,
our proposed method focuses on
integrating cues from both humans and objects for activity recognition,
in a unified optimization framework.

\subsection{Object-Assisted Activity Recognition}

A few methods were implemented to take into account of objects for activity recognition.
Koppula \textit{et al.} \cite{koppula2013learning,koppula2016anticipating} used human and object trajectories as a particle of an anticipatory temporal conditional random field for activity recognition.
Yu \textit{et al.} \cite{yu2014discriminative} combined human body joint and object positions
as representations and used a boosting technique to identify human activities.
Wei \textit{et al.} \cite{wei2013modeling} implemented a human-object interaction model
that combines spatiotemporal human body joint displacements with object recognition and localization in sequence of frames.
Hu \textit{et al.} \cite{hu2015jointly} implemented a heterogeneous method for joint features learning that concatenates spatial displacements of the skeleton data over a sequence of frames, as well as color and depth patterns and their gradients around each joint.
Besides extracting joints and objects as independent features,
probabilistic models based on human poses and object interactions were also designed
for activity recognition \cite{hu2016exemplar}.

Previous object-assisted activity recognition methods
cannot estimate the importance of objects and skeletal features in recognizing
activities.
Several methods \cite{koppula2016anticipating,hu2016exemplar}
also require predefined knowledge about the objects
such as object affordance. 
Our proposed method provides the capability of not only integrating human and object cues,
but also estimate their importance when recognizing human activities.


\section{The Proposed Joint Learning Approach}
\label{sec:approach}


\emph{Notation.}
In this paper, matrices are denoted using boldface uppercase letters, and vectors
using boldface lowercase letters.
For a matrix $\mathbf{M} = \{ m_{ij} \} \in \R^{p \times q}$, we refer to its $i$-th
row as $\mathbf{m}^i$ and its $j$-th column as $\mathbf{m}_j$, and $m_{ij}$ as the
element in the $i$-th row and the $j$-th column.
The Frobenius norm of a matrix $\mathbf{M}$ is computed as
$\|\mathbf{M}\|_F = \sqrt{\sum_{i=1}^p \sum_{j=1}^q m_{ij}^2}$.
For dimensionality, $d_T^j$ represents the dimensionality of the $j$-th body joint of
the human and $d_O^m$ indicates the dimensionality of the $m$-th attribute
modality of an object.

\subsection{Problem Formulation}


We assume that the input data instance set, $\XX = \{ \TT, \OO \}$,
consists of paired observations of a human and  objects.
$\TT = [ \ttt_1, \dots, \ttt_N ] \in \R^{d_T \times N}$ denotes the matrix
of observations of the human, where $\ttt_i \in \R^{d_T}$ is the feature
vector describing the human's $J$ joints in the $i$-th data instance.
Subsections in $\ttt_i$ describe individual joints, with
$\ttt_i^j \in \R^{d_T^j}$ describing the $j$-th joint in the $i$-th data instance.
Each body joint is described by its displacement from the center of the body
(typically, the torso `joint' in many skeletal representations).
Observations of the objects are encoded in the matrix
$\OO = [ \ooo_1, \dots, \ooo_N ] \in \R^{d_O \times N}$, where
$\ooo_i \in \R^{d_O}$ is the feature vector representing all the objects
in the $i$-th data instance.
Each object feature vector is sub-divided to encode multiple objects,
each with $M$ attribute modalities,
such that $\ooo_i^{o_m} \in \R^{d_O^m}$ represents the features describing the
$m$-th attribute modality of the $o$-th object in the $i$-th data instance.

Activity category labels for each training data instance are denoted in the category indicator matrix
$\YY = [ \yyy^1; \dots; \yyy^N ] \in \R^{N \times C}$, 
where $\yyy^i \in \R^C$ denotes the category indicator vector for the $i$-th data instance
and $C$ denotes the number of human activity categories.
Specifically, $y_{ic}$ indicates the probability that the $i$-th
data instance $\mathbf{x}_i = \{\ttt_i, \ooo_i\}$ belongs to the $c$-th activity category.
In the training phase, these probabilities are either $0$ (if the data instance
does not belong to that category) or $1$ (if the data instance belongs to that
category).


We formulate human activity recognition based upon both skeletal observations and object observations as a regression-like optimization problem:
\begin{align}
\label{eq:loss}
\min_{\WW, \UU} & \;\; \|\TT^\top \WW + \OO^\top \UU - \YY \|_F^2
\end{align}
where $\WW = [ \www_1, \dots, \www_C ] \in \R^{d_T \times C}$ represents a weight matrix indicating
the importance of $\TT$ to the activity category labels, and
$\UU = [ \uuu_1, \dots, \uuu_C ] \in \R^{d_O \times C}$ is a weight matrix doing the
same for $\OO$.
$\www_c \in \R^{d_T}$ represents weights of joints with respect to $c$-th category,
with subsections $\www_c^j \in \R^{d_T^j}$ representing the weights
of the $j$-th joint to the $c$-th category.
Similarly, $\uuu_c \in \R^{d_O}$ represents weights of object attributes with respect to $c$-th
category, with subsections
$\uuu_c^{o_m} \in \R^{d_O^m}$ representing weights of the $m$-th attribute of the $o$-th object
to the $c$-th category.

\subsection{Learning Discriminative Joints and Object Attributes}

When recognizing activities, specific body joints and object attributes are typically more discriminative.
For example, joints in the arm are much more important when a human is retrieving an
object from a shelf than joints in the leg would be.
Similarly, attributes describing the object being retrieved would allow a robot
to understand, for instance, whether the human is about to work on a laptop
or read a book.

In order to identify discriminative joints, 
we introduce the \emph{skeletal norm} on the weight matrix $\WW$, defined as: 
\begin{equation}
\label{eq:normS}
\| \WW \|_S = \sum_{c=1}^C \sum_{j=1}^{J} \| \www_c^j \|_2
\end{equation}
This skeletal norm enforces the $\ell_2$-norm within a joint feature and the $\ell_1$-norm between
joints in order to force sparsity and identify discriminative joints (Figure \ref{fig:norms}).
\begin{equation}
\label{eq:with_S}
\min_{\WW, \UU}  \|\TT^\top \WW + \OO^\top \UU - \YY \|_F^2 + \lambda_1 \| \WW \|_S
\end{equation}

Similarly, we also introduce a new \emph{attribute norm}
to learn the importance of various attribute modalities of the objects.
Attribute modalities can describe the color histograms (e.g., red, green, and blue), 
shape (e.g., gradient features), texture, or relationships of the object to the human (e.g., distances).
We define the attribute norm over the weight matrix $\UU$ as:
\begin{equation}
\label{eq:normA}
\| \UU \|_A = \sum_{c=1}^C \sum_{o=1}^{O} \sum_{m=1}^{M} \| \uuu_c^{o_m} \|_2
\end{equation}
The $\ell_2$-norm is employed to enforce similar weights within an attribute modality,
and the $\ell_1$-norm is used between these attribute modalities to enforce sparsity 
in order to identify discriminative object attributes (Figure \ref{fig:norms}).


\begin{figure}
    \centering
    \subfigure[Skeletal Norm]{
        \label{fig:normS}
        \centering
        \includegraphics[height=1.3in]{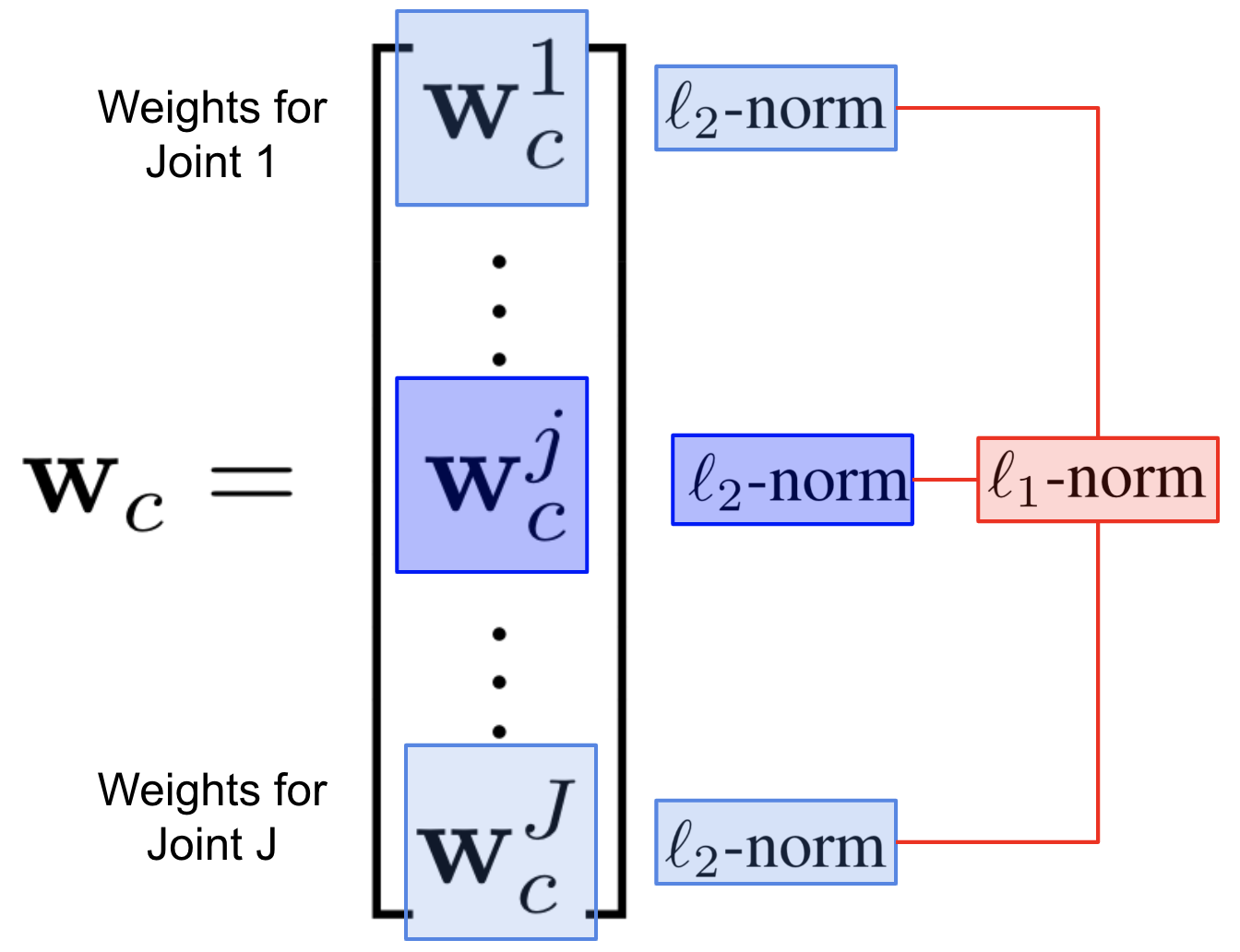}
    }%
    \subfigure[Attribute Norm]{
        \label{fig:normA}
            \centering
        \includegraphics[height=1.3in]{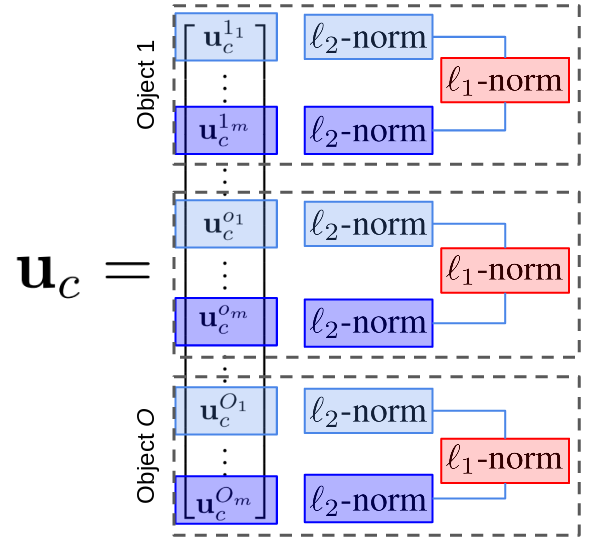}
    }
    \caption{
    Illustrations of the proposed regularization norms.
    Figure \ref{fig:normS} shows the \emph{skeletal norm} $\| \WW \|_S$.
    The $\ell_2$-norm is used within joints and the $\ell_1$-norm is used
    between joints to enforce sparsity and the identification of
    discriminative joints.
    Figure \ref{fig:normA} shows the \emph{attribute norm} $\| \UU \|_A$.
    The $\ell_2$-norm is used within attribute modalities and the $\ell_1$-norm is
    applied between modalities to enforce sparsity and identify
    discriminative attributes.
    } \label{fig:norms}
    \vspace{-6pt}
\end{figure}

Then, our final formulation formulates activity recognition 
as a regularized optimization problem:
\begin{equation}
\label{eq:full}
\min_{\WW, \UU}  \|\TT^\top \WW + \OO^\top \UU - \YY \|_F^2 + \lambda_1 \| \WW \|_S + \lambda_2 \| \UU \|_A
\end{equation}
where $\lambda_1$ and $\lambda_2$ denote the hyperparameters to balance the importance
of the loss function and regularization norms.

\subsection{Recognizing Human Activities}

After we solve the regularized optimization problem in Eq. (\ref{eq:full})
using Algorithm \ref{alg:solution},
we obtain the optimal weight matrices
$\WW^{*} = [ \www_1^{*}, \dots, \www_C^{*} ] \in \R^{d_T \times C}$ and
$\UU^{*} = [ \uuu_1^{*}, \dots, \uuu_C^{*} ] \in \R^{d_O \times C}$.
Each column $\www_c^{*}$ and $\uuu_c^{*}$ denotes the importance of, respectively,
an observation of a human $\ttt$ and objects $\ooo$
to recognize the $c$-th activity category.
Given a new observation $\mathbf{x} = \{\ttt, \ooo\}$, 
the activity category
$y \left( \ttt, \ooo \right)$ is classified by:
\begin{equation}
y \left( \ttt, \ooo \right) = \max_{c} \ttt^\top \www_c^{*} + \ooo^\top \uuu_c^{*}
\end{equation}
Since the objective function in our formulation can be used to perform classification,
no separate classifiers is needed. 

By learning the weight matrices for both human and object observations, 
our approach explicitly identifies discriminative human joints and object attributes.
For example,
consider the learned human observation weight matrix $\WW^{*}$:
\begin{align}
\label{eq:W}
\WW^{*} =
\begin{bmatrix}
\www_1^1 & \cdots & \www_c^1 & \cdots & \www_C^1 \\
\vdots & \ddots & \vdots & & \vdots \\
\www_1^j & \cdots & \www_c^j & \cdots & \www_C^j \\
\vdots & & \vdots & \ddots & \vdots \\
\www_1^J & \cdots & \www_c^J & \cdots & \www_C^J \\
\end{bmatrix}
\end{align}
where $\www_c^j$ represents the importance of
the $j$-th joint to the $c$-th activity category.
The sum-value of all elements within the sub-matrix $\www_c^j$ indicates the relative importance
of the $j$-th human joint when recognizing the $c$-th activity category.
The sum-value of all elements in the row vector $\www^j$
indicates the importance of the $j$-th human body joint to recognize all activity categories.
Similarly, $\UU$ can also be used to analyze and identify which attributes of which objects are important
when recognizing human activities. 

\subsection{Optimization Algorithm}

\begin{algorithm}[tb]
\SetAlgoLined
\SetKwInOut{Input}{Input}
\SetKwInOut{Output}{Output}
\SetNlSty{textrm}{}{:}
\SetKwComment{tcc}{/*}{*/}

\small

\Input{
$\TT = [ \ttt_1, \dots, \ttt_N ] \in \R^{d_{T} \times N}$,
$\OO = [ \ooo_1, \dots, \ooo_N ] \in \R^{d_{O} \times N}$ and
$\YY = [ \yyy^1; \dots; \yyy^N ] \in \R^{N \times C}$.
}
\Output{
$\WW^{*} = \WW \left( i \right) \in \R^{d_T \times C}$ and
$\UU^{*} = \UU \left( i \right) \in \R^{d_O \times C}$.
}
\BlankLine

Let $i = 1$. Initialize $\WW$ and $\UU$ randomly.

\Repeat{convergence}{
Calculate $\mathbf{D}^c_S \left( i + 1 \right)$ for $c \in 1, \dots, C$.

Calculate $\mathbf{D}^c_A \left( i + 1 \right)$ for $c \in 1, \dots, C$.

Calculate $\www_c \left( i + 1 \right)$ via
    Eq. (\ref{eq:deriv_w}) for each $c \in 1, \dots, C$.

Calculate $\uuu_c \left( i + 1 \right)$ via
    Eq. (\ref{eq:deriv_u}) for each $c \in 1, \dots, C$.

$i = i + 1$.
}

\Return $\WW^{*}$ and $\UU^{*}$

\caption{An iterative algorithm to solve the formulated optimization
problem in Eq. (\ref{eq:full}).}
\label{alg:solution}
\end{algorithm}

The formulated optimization problem in Eq. (\ref{eq:full}) is difficult to solve  
because the regularization norms $\|\WW\|_S$ and $\|\UU\|_A$ are not smooth and because we need
to simultaneously find the optimal solutions for $\WW$ and $\UU$, both of which the final
solution depends on.
To solve this, we propose a new iterative optimization solver as presented in Algorithm \ref{alg:solution}.

We calculate the derivative of Eq. (\ref{eq:full}) with respect to $\www_c$
in order to solve $\WW$:
\begin{equation}
\TT \TT^\top \www_c + \TT \OO^\top \uuu_c - \TT \yyy_c + \lambda_1 \mathbf{D}^c_S \www_c = \mathbf{0}
\end{equation}
\vspace{-9pt}
\begin{equation}
\label{eq:deriv_w}
\www_c = \left( \TT \TT^\top + \lambda_1 \mathbf{D}^c_S \right)^{-1} \TT \left( \yyy_c - \OO^\top \uuu_c \right)
\end{equation}
where $\mathbf{D}^c_S$ is a block diagonal matrix with
$\frac{1}{2 \| \www_c^j \|_2} \mathbf{I}_{d_T^j}$ as the $j$-th block.

Similarly, we compute the derivative of Eq. (\ref{eq:full}) with respect to $\uuu_c$
in order to solve $\UU$:
\begin{equation}
\OO \OO^\top \uuu_c + \OO \TT^\top \www_c - \OO \yyy_c + \lambda_2 \mathbf{D}^c_A \uuu_c = \mathbf{0}
\end{equation}
\vspace{-9pt}
\begin{equation}
\label{eq:deriv_u}
\uuu_c = \left( \OO \OO^\top + \lambda_2 \mathbf{D}^c_A \right)^{-1} \OO \left( \yyy_c - \TT^\top \www_c \right)
\end{equation}
where $\mathbf{D}^c_A$ is a block diagonal matrix having $O$ blocks.
Each of these diagonal blocks is composed of $M$ diagonal blocks, where the $m$-th
diagonal block is
$\frac{1}{2 \| \uuu_c^{o_m} \|_2} \mathbf{I}_{d_O^m}$.

Because the solution to $\www_c$ depends on both $\uuu_c$ and $\mathbf{D}^c_S$ (which
is dependent on $\www_c$), and the solution to $\uuu_c$ depends on both $\www_c$
and $\mathbf{D}^c_A$ (which is dependent on $\uuu_c$), an iterative optimization
algorithm is necessary to address this problem.
The proposed optimization solver is detailed in Algorithm \ref{alg:solution},
which alternately solving $\www_c$ and $\uuu_c$ until convergence.
This proposed algorithm holds a theoretical guarantee to converge to the optimal solution:

\begin{theorem}\label{thm1}
Algorithm 1 is guaranteed to converge to the optimal solution to the formulated regularized optimization problem in Eq. (\ref{eq:full}).
\end{theorem}

\begin{proof}
See supplementary materials\footnote{\url{hcr.mines.edu/publication/HAR_Supp.pdf}}.
\end{proof}

The time complexity of Algorithm \ref{alg:solution} is dominated by Steps (5) and (6), because
Steps (3) and (4) are trivial, executing in linear time of
$\mathcal{O} \left( C d_T \right)$ and
$\mathcal{O} \left( C d_O \right)$, respectively.
Steps (5) and (6) can be solved as a system of linear equations instead
of performing the matrix inverse, and are respectively
$\mathcal{O} \left( {d_T}^2 \right)$ and $\mathcal{O} \left( {d_O}^2 \right)$.

\section{Experimental Results}
\label{sec:results}


We assess our approach's performance
on two benchmark activity recognition datasets and using a physical robot as a case study.
In the experiments,
we used one type of skeletal features and three
different object attribute modalities.
The skeletal feature used is the displacement of each body joint
from the central torso joint.
The object attributes used are color, shape, and object-joint distance.
The color attribute is implemented using red, green, and blue histograms.
The shape attribute is implemented using the Histogram of Oriented Gradients (HOG) \cite{dalal2005histograms} features.
Finally, the object-joint distance attribute is implemented through calculating the
distance in 3D space from the object to each skeletal joint
to model the interaction between objects and the human.
We also evaluate our case study using a multinomial probability
distribution as the object attributes, showing our method's ability to
identify the most important objects with respect to specific human activities.

\begin{figure}[h]
    \centering
    \includegraphics[width=0.44\textwidth]{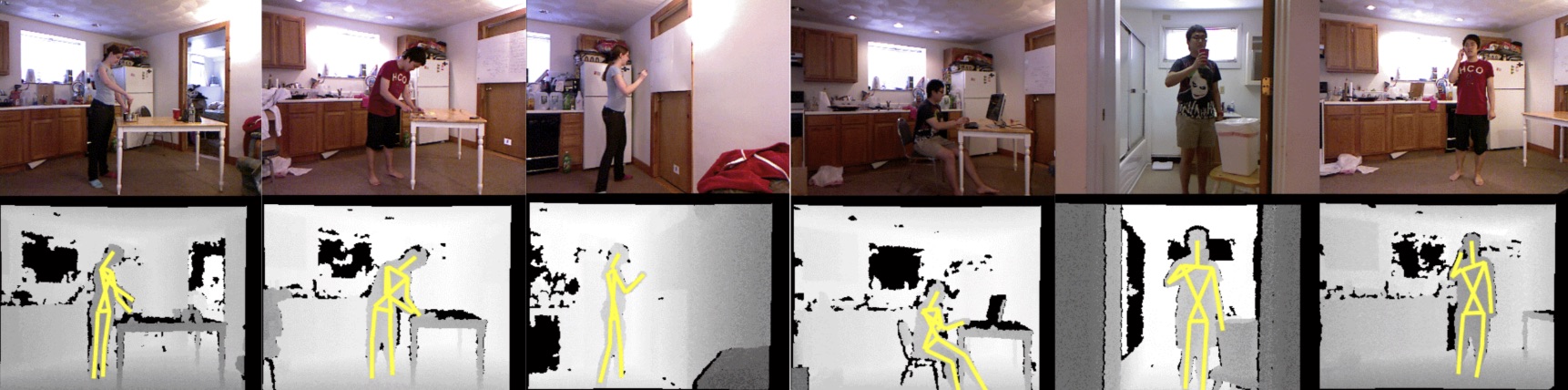}    
    \vspace{-4pt}
    \caption{Example images from the CAD-60 dataset.}
    \label{fig:cad_activities}
    \vspace{-6pt}
\end{figure}

\subsection{Results on Cornell Activity Dataset}

The Cornell Activity Dataset (CAD-60) has been widely used as
a standard benchmark dataset for activity recognition \cite{sung2012unstructured} in robotics.
The dataset consists of activities performed by 4 humans, with
each activity execution recorded as color images, depth images, and
annotated skeleton joint positions for 15 joints.
Our experiments used six activities that involve
the use of objects, including \emph{writing on whiteboard},
\emph{cooking (stirring)},
\emph{cooking (chopping)},
\emph{working on computer}, \emph{rinsing mouth with water}, and \emph{talking on the phone},
as illustrated in Figure \ref{fig:cad_activities}.
All objects existing in a scene were utilized,
even if they did not relate to the human
activity. For example, the whiteboard was still included in our experiments
when recognizing the cooking (stirring) human activity scenes.

\begin{table}[t]
    \centering
    \caption{Accuracy obtained by our approach on the CAD-60 dataset and comparisons
    to previous approaches.}
    \label{tab:cad_accuracy}
    \vspace{-4pt}
    \tabcolsep=0.3cm
    \begin{tabular}{|c|c|}
    \hline
    Approach & Accuracy \\
    \hline\hline
    Feature and Body Part Learning \cite{han2017simultaneous} & 83.93\% \\
    Joint Heterogeneous Features Learning \cite{hu2015jointly} & 84.10\% \\
    Spatiotemporal Interest Point \cite{zhu2014evaluating} & 87.50\% \\
    Feature-Level Fusion \cite{zhu2015fusing} & 87.50\% \\
    Pose Kinetic Energy \cite{shan20143d} & 91.90\% \\
    Sparse Coding Dictionary Learning \cite{qi2018learning} & 94.12\% \\
    Kinect + Pose machine \cite{das2017action} & 95.58\% \\
    \hline
    Our Approach (only \emph{skeletal norm}) & 86.86\% \\
    Our Approach (only \emph{attribute norm}) & 96.18\% \\
    \textbf{Our Approach} & \textbf{98.11\%} \\
    \hline
    \end{tabular}
    \vspace{-6pt}
\end{table}

The quantitative experimental results are listed in Table \ref{tab:cad_accuracy}.
It is observed that the proposed approach is able to identify
$98.11\%$ of these activity executions correctly.
This table also compares our results with other state-of-the-art methods
for human activity recognition, 
which indicates that our simultaneous
learning from both human observations and object observations
provides superior performance.
This table also shows that limiting our approach to only one of our
proposed sparsity inducing norms degrades our performance.
This drop off is significant when we only deoploy the \emph{skeletal norm} in the formulation,
showing that object attributes provide useful discriminative
information.

\subsection{Results on MSR Activity 3D Dataset}

We further evaluate our approach based on  the MSR Daily Activity 3D Dataset
\cite{li2010action}, a commonly used public dataset for benchmarking human
activity recognition approaches.
This dataset consists of daily human activities performed by ten human subjects,
which is recorded with color images, depth images, and annotated joint positions
for 20 skeletal joints.
Examples of color and depth images from activity categories involving the use of objects
are illustrated in Figure \ref{fig:msr_activities}.

\begin{figure}[h]
\centering
    \includegraphics[width=0.44\textwidth]{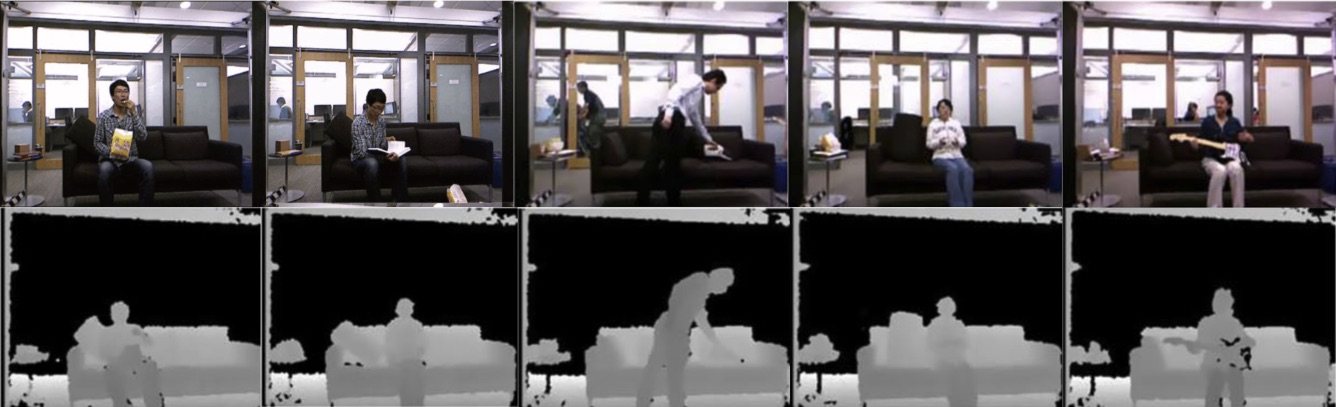}
    \vspace{-4pt}
    \caption{Example images from the MSR Daily Activity 3D dataset.}
    \label{fig:msr_activities}
    \vspace{-3pt}
\end{figure}

The quantitative experimental results are listed in Table \ref{tab:msr_accuracy}.
It is observed that our approach achieves an activity recognition accuracy of $97.71\%$.
Table \ref{tab:msr_accuracy} also compares our results with
existing state-of-the-art approaches that have also been
tested on this dataset, showing again that our simultaneous learning
from both human body joints and object attributes is effective
to recognize human activities.
Moreover, the results indicate the importance of the proposed
\emph{attribute norm}, as omitting this norm causes
our approach to loss the ability of identifying most important object attributes,
thus decreasing the recognition accuracy.

\begin{table}[htp]
    \centering
    \caption{Accuracy obtained by our approach on the MSR Daily Activity 3D dataset and comparisons to previous approaches.}
    \label{tab:msr_accuracy}
    \vspace{-4pt}
    \tabcolsep=0.35cm
    \begin{tabular}{|c|c|}
    \hline
    Approach & Accuracy \\
    \hline\hline
    Sparse Coding Dictionary Learning \cite{qi2018learning} & 68.75\% \\
    BIPOD \cite{reily2018skeleton} & 79.70\% \\
    Spatiotemporal Interest Point \cite{zhu2014evaluating} & 80.00\% \\
    Key-Pose-Motifs \cite{Wang_2016_CVPR} & 83.47\% \\
    Kinect + Pose machine \cite{das2017action} & 84.37\% \\
    Feature-Level Fusion \cite{zhu2015fusing} & 88.80\% \\
    3D joint+CS-Mltp (concatenate) \cite{luo2014spatio} & 92.50\% \\
    DL-GSGC \cite{luo2013group} & 95.00\% \\
    Joint Heterogeneous Features Learning \cite{hu2015jointly} & 95.00\% \\
    $\tau -test$ \cite{lu2014range} & 95.63\% \\
    \hline
    Our Approach (only \emph{skeletal norm}) & 82.00\% \\
    Our Approach (only \emph{attribute norm}) & 95.71\% \\
    \textbf{Our Approach} & \textbf{97.71\%} \\
    \hline
    \end{tabular}
\end{table}

\subsection{Results in Home Assistance Case Studies}

\begin{figure}
    \centering
    \subfigure[TurtleBot Platform]{
        \label{fig:platform}
        \centering
        \includegraphics[height=1.45in]{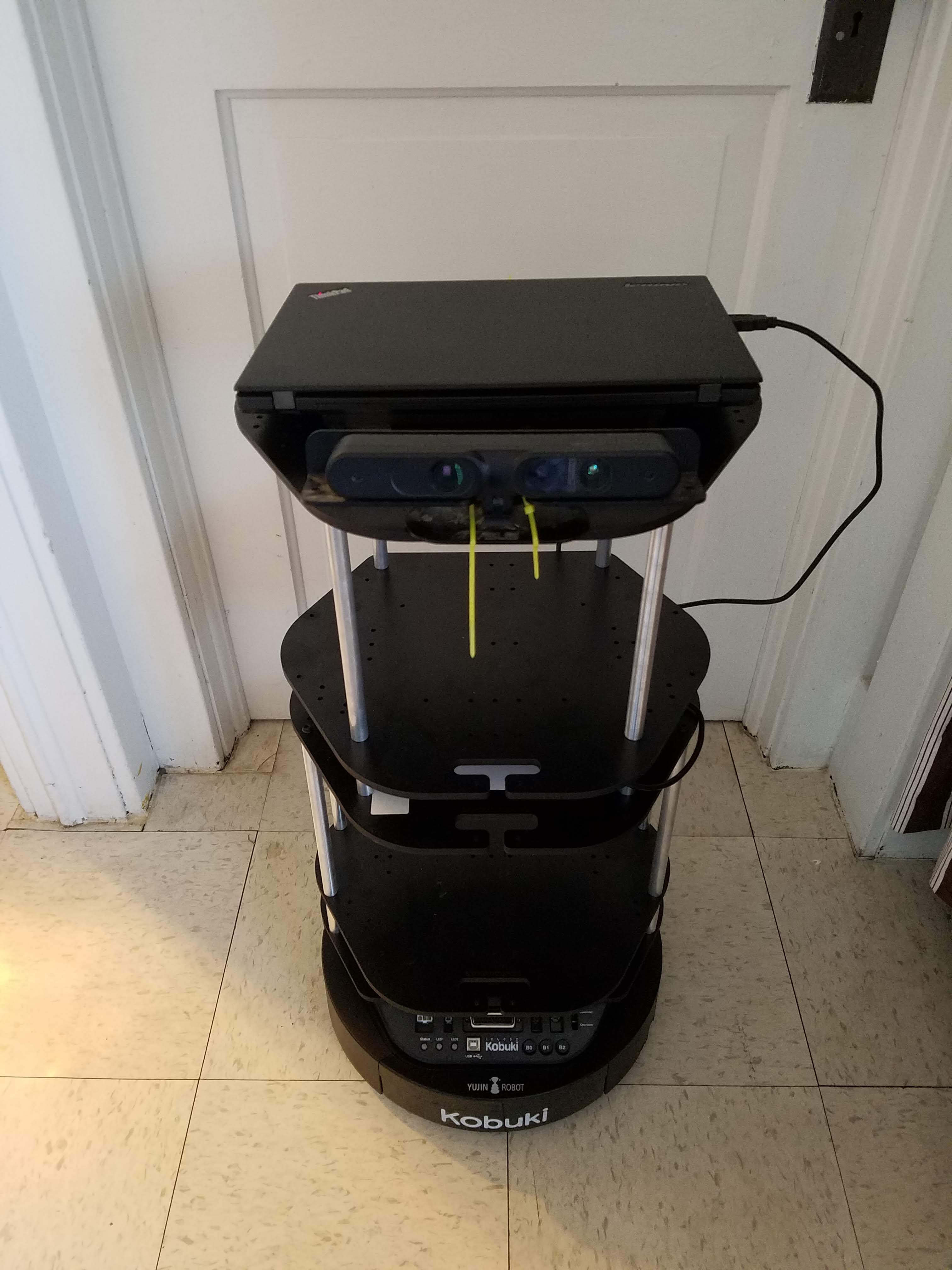}
    }%
    \subfigure[Scenario Setup]{
        \label{fig:recording}
        \centering
        \includegraphics[height=1.45in]{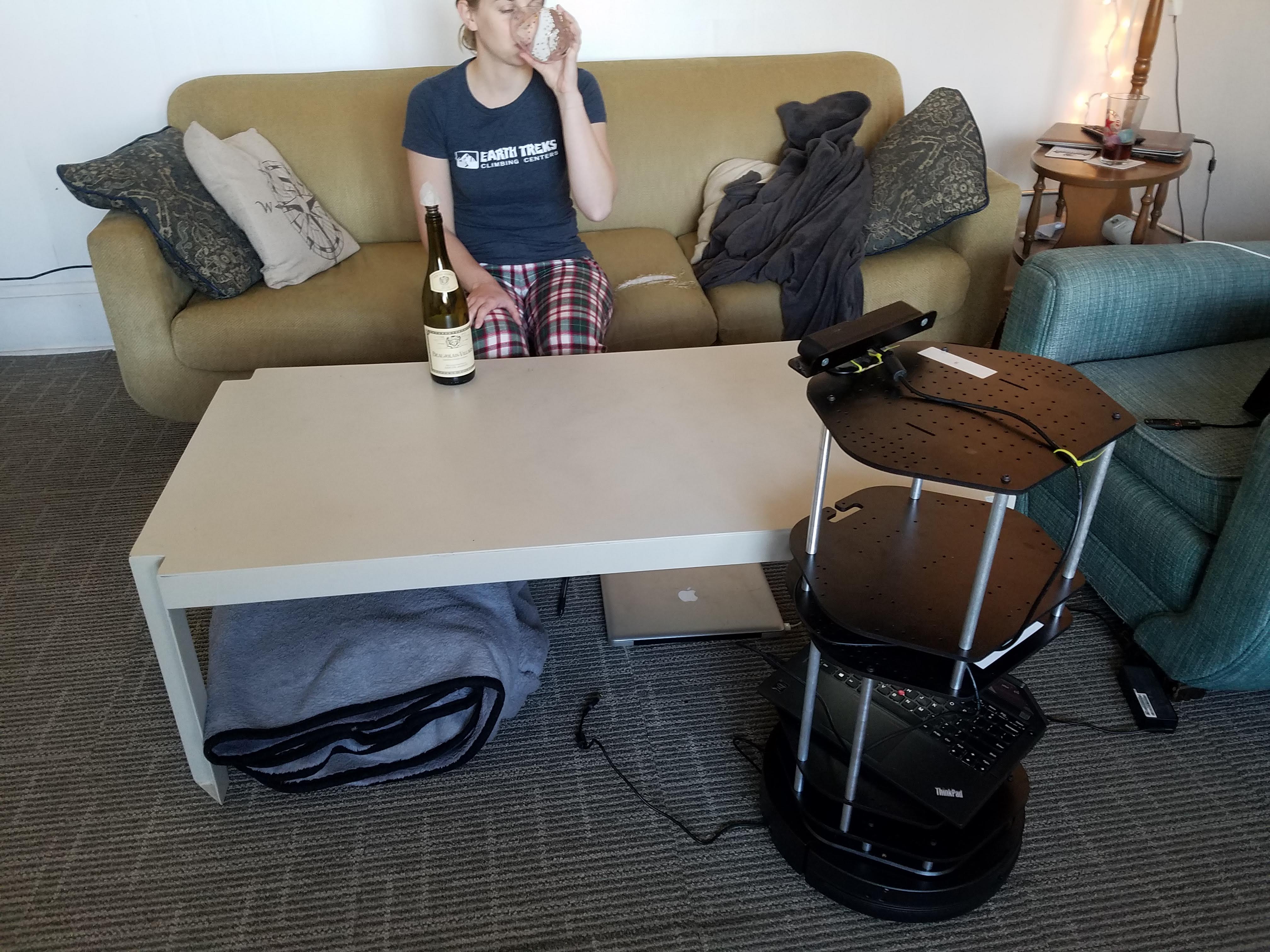}
    }
    \vspace{-4pt}
    \caption{The experiment setup used in the case studies.
    Figure \ref{fig:platform} shows the Turtlebot platform
    equipped with an ASUS Xtion Pro camera and laptop for computation.
    Figure \ref{fig:recording} shows the robot observing human activities.
    }
    \label{fig:setup}
    \vspace{-6pt}
\end{figure}

\begin{figure}
    \centering
    \includegraphics[width=0.44\textwidth]{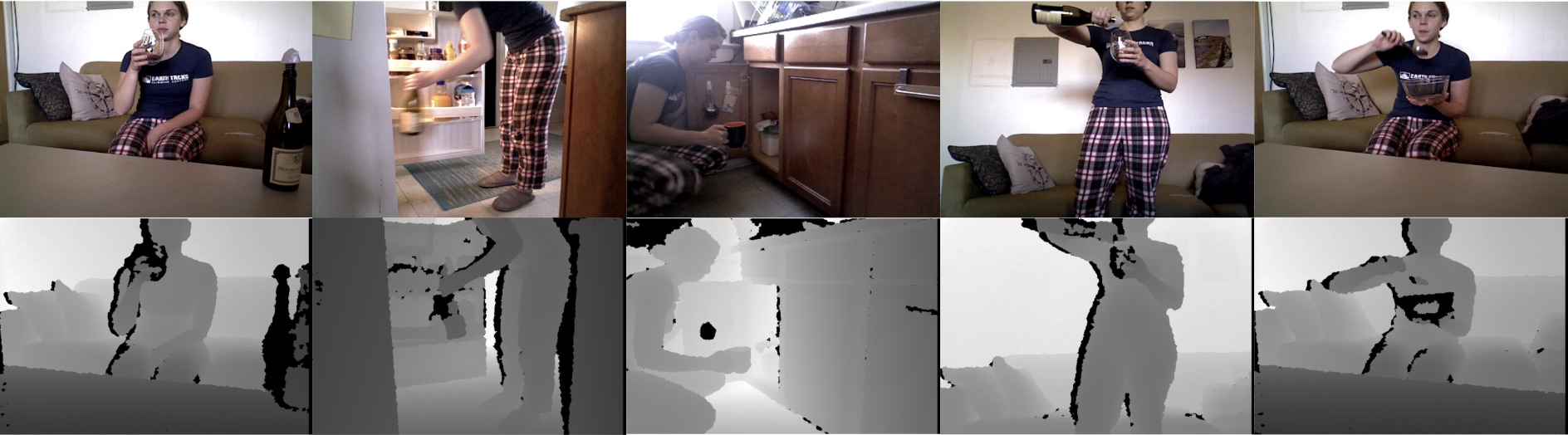}
    \vspace{-4pt}
    \caption{Color and depth images of the activities included
    in our case studies in a simulated home assistance scenario.
    From left to right, \emph{drinking wine}, \emph{storing food},
    \emph{storing dishes}, \emph{pouring wine}, and \emph{eating}.
    }
    \label{fig:home_activities}
\end{figure}

In addition to evaluating our method on the publicly available
benchmark datasets, we also implemented our approach on a physical
robot in order to demonstrate its performance in case study.
We deployed our approach on a Turtlebot robot participating in
a simulated home assistance scenario (Figure \ref{fig:setup}).
The robot has a color-depth sensor onboard
to extract 3D skeleton data and a lightweight netbook for processing.

In this scenario, five activities were defined, with Figure \ref{fig:home_activities}
showing example color and depth images for each activity.
These activities are \emph{drinking wine}, \emph{storing food},
\emph{storing dishes}, \emph{pouring wine}, and \emph{eating}.
Each activity was performed 20 times.
In order to test our approach in learning simultaneously
from observations of the human and the objects, these
activities were defined to involve similar objects and human poses.
For example, both drinking wine and pouring wine involve a
glass and a bottle; however, drinking wine is
performed while sitting down and pouring wine is performed while standing up.
Similarly, both eating and drinking wine are
activities performed by a sitting human,
but they involve different objects (respectively, a bowl and a spoon
versus a wine glass and a bottle).

The quantitative experimental results are presented in Table \ref{tab:case_accuracy}.
We can observe that the proposed approach achieves an overall accuracy of $98.33\%$ in
the case studies.
Comparison with baseline real-time approaches is also listed in Table \ref{tab:case_accuracy},
which shows that our approach is superior to
two standard real-time machine learning methods as baselines.
With only the \emph{skeletal norm} or the \emph{attribute norm} as the regularization,
our approach achieves good accuracy but less than with the
complete formulation that uses both norms.

\begin{table}[tb]
    \centering
    \caption{Accuracy obtained by our approach in the case studies and comparison to baseline real-time approaches.}
    \vspace{-4pt}
    \label{tab:case_accuracy}
    \tabcolsep=0.35cm
    \begin{tabular}{|c|c|}
    \hline
    Approach & Accuracy \\
    \hline\hline
    Support Vector Machine & 51.67\% \\
    Decision Forest & 91.67\% \\
    \hline
    Our Approach (only \emph{skeletal norm}) & 95.00\% \\
    Our Approach (only \emph{attribute norm}) & 96.67\% \\
    Our Approach & \textbf{98.33\%} \\
    \hline
    \end{tabular}
    \vspace{-6pt}
\end{table}

We also tested our method with a different
set of attributes in order to assess its ability to identify
discriminative objects.
In this setup, we defined five attribute modalities, where each modality
is the probability that an object category appeared in the view of the robot.
The 5 object categories used are the \emph{wine bottle}, \emph{glass},
\emph{fridge}, \emph{bowl}, and \emph{spoon}.
The probabilities that an object appeared in a scene were obtained from the YOLO object detection system \cite{redmon2018yolov3},
which uses a pre-trained neural network to
identify common household objects.
For example, for the activity of drinking wine, the probability of a
bottle or glass appearing would be close to $1$, and close
to $0$ for the remaining objects.

\begin{figure}[h]
    \centering
     \vspace{-6pt}
    \subfigure[Drinking Wine]{
        \label{fig:u_drinking_wine}
        \centering
        \includegraphics[width=0.2\textwidth]{{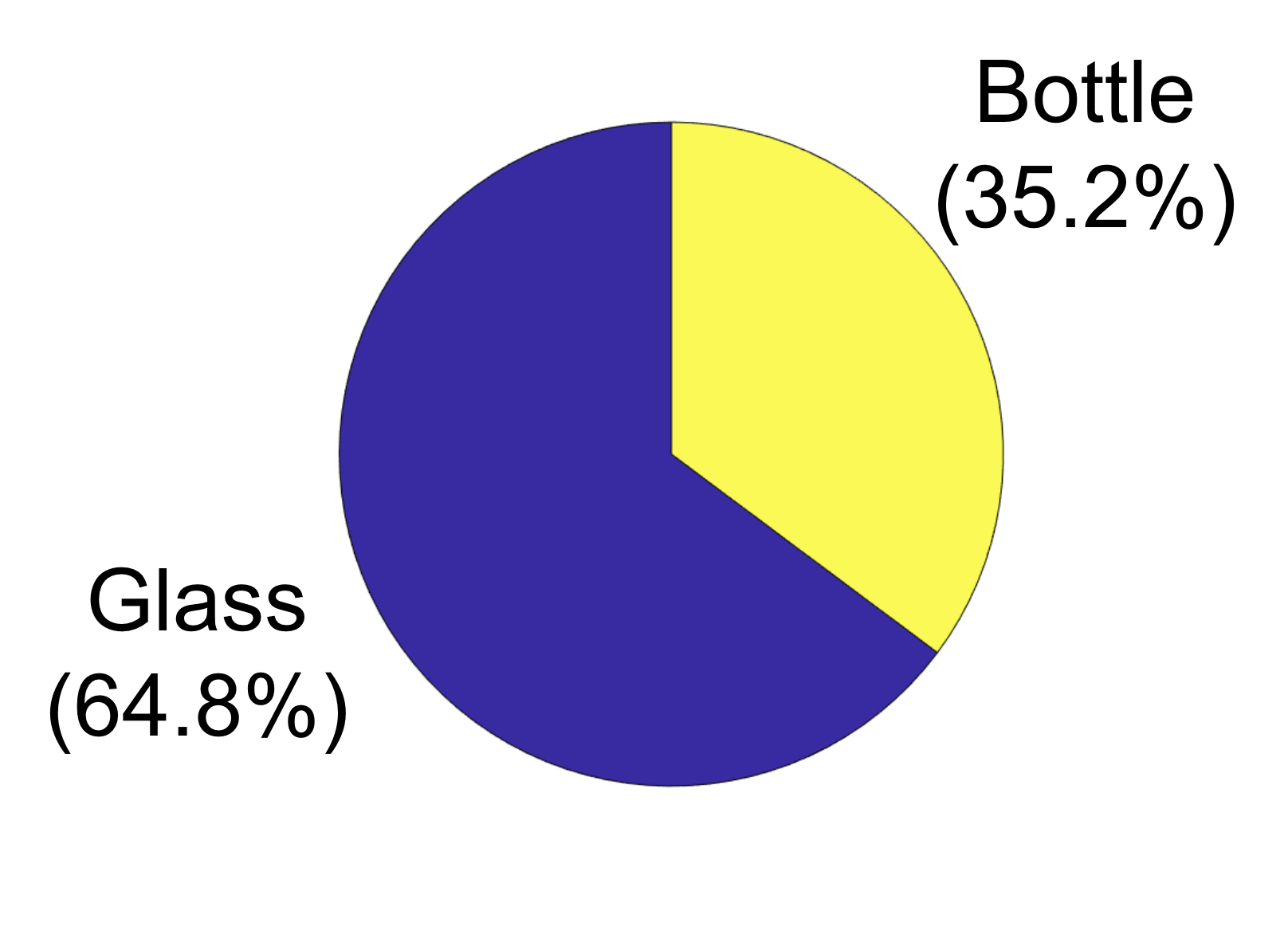}}
    }%
    \subfigure[Storing Food]{
        \label{fig:u_storing_food}
        \centering
        \includegraphics[width=0.2\textwidth]{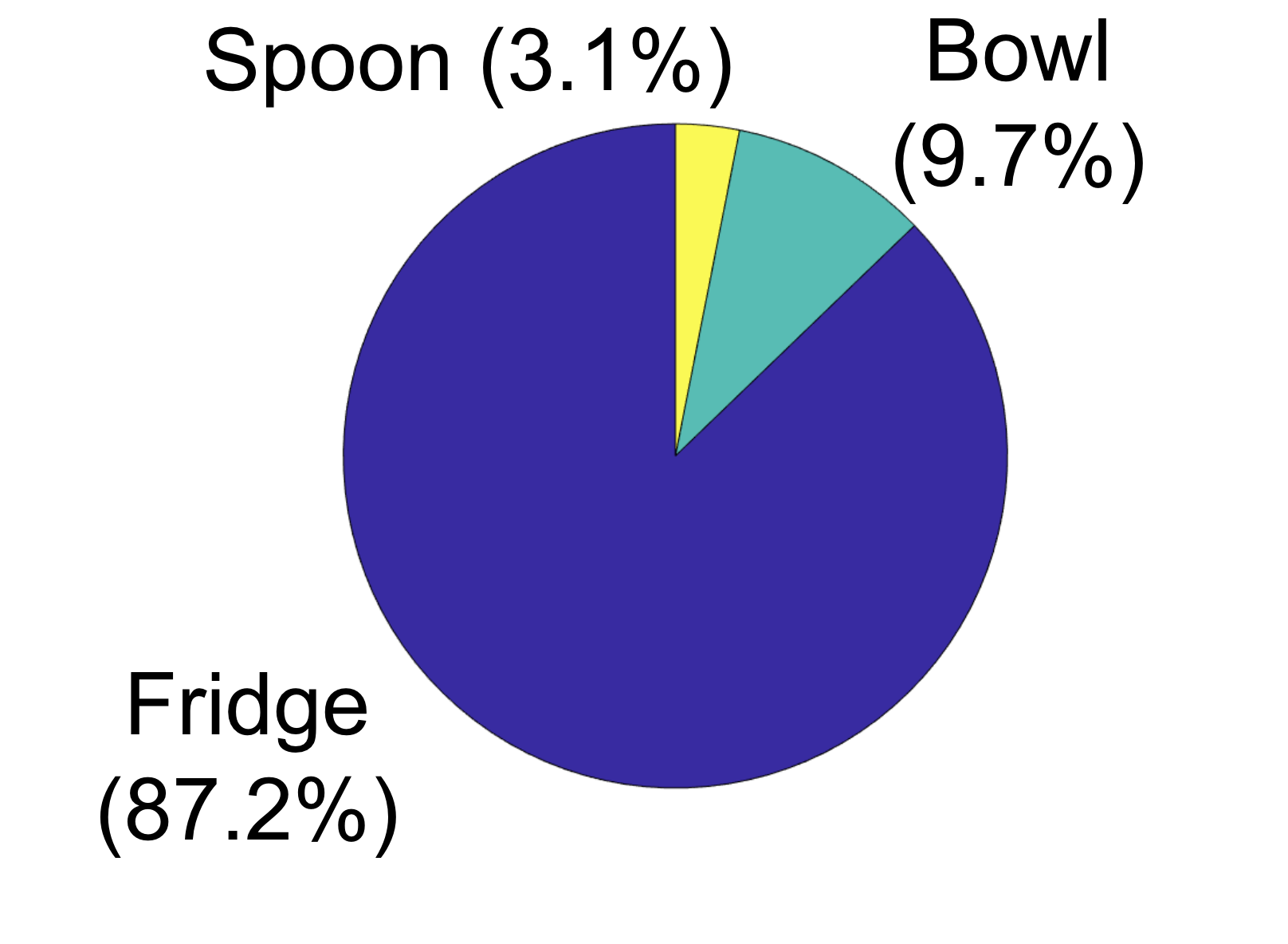}
    }
    \vspace{-6pt}
    \caption{Illustration of the distribution of weights in
    two columns of the weight matrix $\UU$.
    Figure \ref{fig:u_drinking_wine} demonstrates the weights for the
    drinking wine activity, where the glass and
    bottle are important.
    Figure \ref{fig:u_storing_food} illustrates the weights for the human activity of
    storing food, where the fridge is the most relevant object.
    }
    \label{fig:u}
     \vspace{-3pt}
\end{figure}

Using this setup, our approach is able to recognize $96.67\%$ of home
activities correctly.
Additionally, this setup allowed our approach to identify discriminative
objects, as each column of the $\UU$ matrix contained only five values,
each relating one object to that column's associated human activity.
Figure \ref{fig:u} displays two columns from the $\UU$ weight matrix.
In Figure \ref{fig:u_drinking_wine}, we observe that the bottle and
glass are the only objects receiving weights, as these are
very indicative of the drinking wine activity.
Similarly in Figure \ref{fig:u_storing_food}, we see that the fridge
receives nearly $90\%$ of the total column weight, identifying
it as being very indicative to recognize the storing food activity.
Similarly, Figure \ref{fig:u_obj} displays this relationship
for two of the rows of $\UU$.
Figure \ref{fig:u_bowl} demonstrates that for the bowl object,
the most related activity is eating.
Figure \ref{fig:u_fridge} shows this for the fridge
object, which shows that the most relevant activity is storing food,
the only activity category in which this object appears.

\begin{figure}
    \centering
    \subfigure[Bowl]{
        \label{fig:u_bowl}
        \centering
        \includegraphics[width=0.2\textwidth]{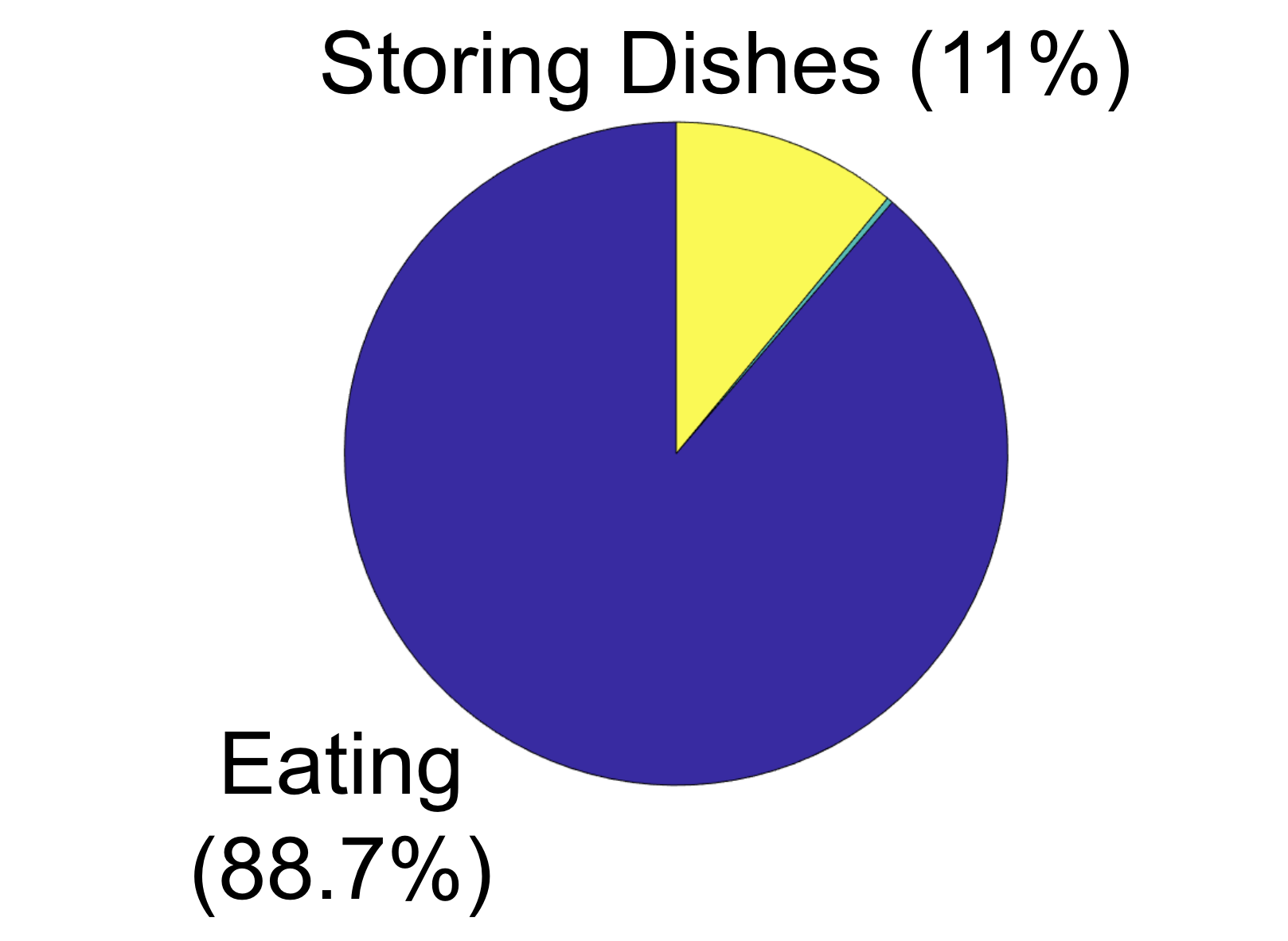}
    }%
    \subfigure[Fridge]{
        \label{fig:u_fridge}
        \centering
        \includegraphics[width=0.2\textwidth]{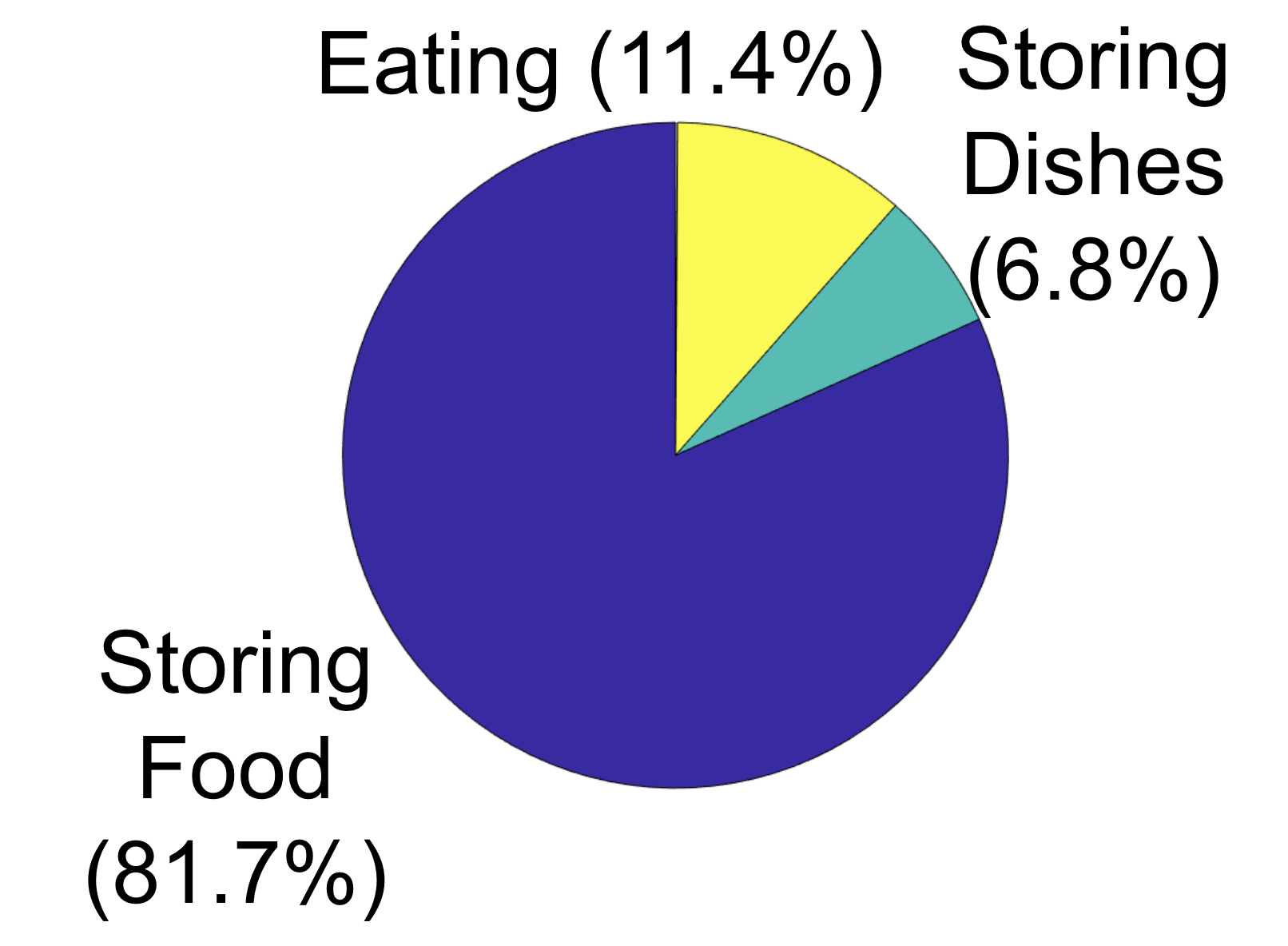}
    }
    \vspace{-6pt}
    \caption{Illustration of the distribution of weights in
    two rows of the weight matrix $\UU$.
    Figure \ref{fig:u_bowl} illustrates the weights for the
    bowl object, where the eating
    is the most relevant activity.
    Figure \ref{fig:u_fridge} depicts the weights for the
    fridge object, where storing food
    is the most relevant activity.
    }    \label{fig:u_obj}
        \vspace{-6pt}
\end{figure}

\subsection{Discussion}

\subsubsection{Real-Time Performance}
One of the major advantages of our approach
is its ability to run at real-time speeds.
For each dataset we evaluated, we analyzed the runtime of our
activity recognition approach, summarized in Table \ref{tab:real-time}.
Due to the efficiency of our proposed convex problem formulation
that integrates classification
within the loss function,
our approach obtains recognition processing speeds in excess of $2 \times 10^4$ Hz,
when executing on an Intel i5 processor with 4Gb memory.
Our approach provides a suitable solution for accurate and real-time recognition
of human activities.

\begin{table}[h]
    \centering
    \caption{Experimental results on real-time performance.}
    \begin{tabular}{|c|c|c|c|}
    \hline
    Metric & CAD-60 & MSR & Home Asst. \\
    \hline\hline
    Processing Speed (Hz) & $6.1 \times 10^4$ & $2.5 \times 10^4$ & $2.2 \times 10^4$ \\
    Time Per Frame (sec) & $1.6 \times 10^{-5}$ & $3.9 \times 10^{-5}$ & $4.5 \times 10^{-5}$ \\
    \hline
    \end{tabular}
    \label{tab:real-time}
\end{table}

\subsubsection{Hyperparameter Selection}

In our problem formulation in Eq. (\ref{eq:full}),
the hyperparameters
$\lambda_1$ and $\lambda_2$ control the importance of the
\emph{skeletal norm} and \emph{attribute norm}, respectively,
and balance these two norms with the loss function.
As our presented results have demonstrated,
the method's accuracy
decreases with either of these hyperparameters assigned to $0$.
As each norm captures different information (i.e., weights of skeletal features
or object attributes),
both are necessary for the proposed approach to achieve its
state-of-the-art accuracy.
Also, we observe that as the values of these hyperparameters
become too large, performance decreases as the loss function
relating observations to activity labels becomes less
important.
Our presented results use the hyperparameter values of
$\lambda_1 = 0.1$ and $\lambda_2 = 0.1$.

\section{Conclusion}
\label{sec:conclusion}

In this paper, we introduce a principled method for activity recognition
through simultaneous learning from observations of human poses and object attributes.
The proposed approach is capable of identifying discriminative
joints and object attributes, providing an interpretable understanding
of their importance to human activities and the relationships between
joints and objects.
We formulate activity recognition as an optimization problem
that uses a regression-like loss function to integrate teammate and object
cues to perform activity recognition,
and utilizes sparsity-inducing norms to estimate feature importance.
We introduce an iterative algorithm guaranteed to find the optimal solution.
We assess our proposed approach on two benchmark activity
recognition datasets,
and on an actual robot to show a case study.
Experimental results have shown that our approach achieves 
state-of-the-art recognition accuracy and obtains real-time performance.

\bibliographystyle{ieeetr}
\bibliography{references}

\end{document}


\newtheorem{definition}{Definition}
\newtheorem{theorem}{\textbf{Theorem}}
\newtheorem{lemma}{\textbf{Lemma}}
\newtheorem{proposition}{Proposition}
\newtheorem{property}{Property}
\newtheorem{observation}{Observation}
\newtheorem{corollary}{\textbf{Corollary}}

\maketitle
\thispagestyle{empty}
\pagestyle{empty}

\section{Proof of Theorem 1}

This section provides a mathematical proof that Algorithm 1 within the main paper decreases the value of the formulated objective function in each iteration and also converges to the optimal value.

First, we present a lemma:
\begin{lemma}\label{lemma1}
For any two given vectors $\mathbf{v}$ and $\mathbf{\tilde{v}}$, the following inequality relation holds:
$\|\mathbf{\tilde{v}}\|_2 - \frac{\|\mathbf{\tilde{v}}\|_2^2}{2\|\mathbf{v}\|_2}
\leq
\|\mathbf{v}\|_2 - \frac{\|\mathbf{v}\|_2^2}{2\|\mathbf{v}\|_2}
$.
\end{lemma}

\begin{proof}
We have:
\begin{eqnarray}
\left(\Vert\widetilde{\mathbf{v}}\Vert_{2}-\Vert\mathbf{v}\Vert_{2}\right)^2 \geq 0
\end{eqnarray}
\begin{eqnarray}
\Vert\widetilde{\mathbf{v}}\Vert_{2}^{2} -  2\Vert\widetilde{\mathbf{v}}\Vert_{2}\Vert\mathbf{v}\Vert_{2} + \Vert\mathbf{v}\Vert_{2}^{2} \geq 0
\end{eqnarray}
\begin{eqnarray}
\Vert\widetilde{\mathbf{v}}\Vert_{2} - \frac{\Vert\widetilde{\mathbf{v}}_2^2}{2\Vert\mathbf{v}\Vert_2} \leq \frac{\Vert\mathbf{v}\Vert_2}{2}
\end{eqnarray}
\begin{eqnarray}
\Vert\widetilde{\textbf{v}}\Vert_{2} - \dfrac{\Vert\widetilde{\textbf{v}}\Vert_{2}^{2}}{2\Vert\textbf{v}\Vert_{2}} \leq  \Vert\textbf{v}\Vert_{2} - \dfrac{\Vert\textbf{v}\Vert_{2}^{2}}{2\Vert\textbf{v}\Vert_{2}}
\end{eqnarray}
\end{proof}

Then, we prove Theorem 1 in the main paper:

\begin{theorem}\label{thm1}
Algorithm 1 is guaranteed to converge to the optimal solution to the formulated regularized optimization problem in Eq. (5) (in the main paper).
\end{theorem}

\emph{Proof}:
In each iteration, Algorithm 1 alternately calculates $\www_c$ (in Step 5 of Algorithm 1) and $\uuu_c$ (in Step 6).
In Steps 3 and 4 respectively, we update the block matrices $\mathbf{D}_S^c$ and $\mathbf{D}_A^c$.
$\mathbf{D}_S^c \left( i + 1 \right)$ contains $J$ blocks, where the $j$-th block is $\frac{1}{2 \| \www_j^c \left( i \right) \|_2} \II_{d_T^j}$, 
and $\mathbf{D}_A^c \left( i + 1 \right)$ contains $O$ blocks, each of which
contains $M$ blocks, where the $m$-th block is
$\frac{1}{2 \| \uuu_c^{o_m} \left( i \right) \|_2} \II_{d_O^m}$.
These block structures are based on the derivatives of our skeletal and attribute norms, each of which acts on subsections of a vector.

\vspace{6pt}
\noindent\textbf{Update $\WW\left( i + 1 \right)$ in Step 5.}
After updating $\mathbf{D}_S^c$ in Step 3, we update $\WW \left( i + 1 \right)$ by treating $\UU$ as fixed:
\begin{align}
\WW \left( i + 1 \right) = \min_{\WW} \| \TT^\top \WW + \OO^\top \UU - \YY \|_F^2 \notag \\
+ \lambda_1 \sum_{c=1}^{C} \www_c^\top \mathbf{D}_S^c \left( i + 1 \right) \www_c 
\end{align}

We define $\mathcal{J}_W \left( i \right) = \| \TT^\top \WW \left( i \right) + \OO^\top \UU - \YY \|_F^2$ to write
\begin{align}
\mathcal{J}_W \left( i + 1 \right) + \lambda_1 \sum_{c=1}^{C} \www_c^\top \left( i + 1 \right) \mathbf{D}_S^c \left( i + 1 \right) \www_c \left( i + 1 \right) \notag \\
\leq \mathcal{J}_W \left( i \right) + \lambda_1 \sum_{c=1}^{C} \www_c^\top \left( i \right) \mathbf{D}_S^c \left( i \right) \www_c \left( i \right)
\end{align}

By the definition of $\mathbf{D}_S^c$,
\begin{align}
\label{eq:W_defD}
\mathcal{J}_W \left( i + 1 \right) + \lambda_1 \sum_{c=1}^{C} \sum_{j=1}^{J} \frac{\| \www_c^j \left( i + 1 \right) \|_2^2}{2 \| \www_c^j \left( i \right) \|_2} \notag \\
\leq \mathcal{J}_W \left( i \right) + \lambda_1 \sum_{c=1}^{C} \sum_{j=1}^{J} \frac{\| \www_c^j \left( i \right)}{2 \| \www_c^j \left( i \right) \|_2}
\end{align}
By Lemma \ref{lemma1}, for $c \in 1, \dots, C$ we can define that
\begin{align}
\label{eq:W_from_lemma}
\sum_{j=1}^{J} \| \www_c^j \left( i + 1 \right) \|_2 - \sum_{j=1}^{J} \frac{\| \www_c^j \left( i + 1 \right) \|_2^2}{2 \| \www_c^j \left( i \right) \|_2} \notag \\
\leq \sum_{j=1}^{J} \| \www_c^j \left( i \right) \|_2 - \sum_{j=1}^{J} \frac{\| \www_c^j \left( i \right) \|_2^2}{2 \| \www_c^j \left( i \right) \|_2}
\end{align}
By adding Eqs. (\ref{eq:W_defD}) and (\ref{eq:W_from_lemma}), we obtain
\begin{align}
\label{eq:Step5Proof}
\mathcal{J}_W \left( i + 1 \right) + \lambda_1 \sum_{c=1}^{C} \sum_{j=1}^{J} \| \www_c^j \left( i + 1 \right) \|_2 \notag \\
\leq \mathcal{J}_W \left( i \right) + \lambda_1 \sum_{c=1}^{C} \sum_{j=1}^{J} \| \www_c^j \left( i \right) \|_2
\end{align}

\vspace{6pt}
\noindent\textbf{Update $\UU\left( i + 1 \right)$ in Step 6.}
Similar to the proof for $\WW\left( i + 1 \right)$, after updating $\mathbf{D}_A^c$ in Step 4,
we update $\UU \left( i + 1 \right)$, this time treating $\WW$ as fixed:
\begin{align}
\UU \left( i + 1 \right) = \min_{\UU} \| \TT^\top \WW + \OO^\top \UU - \YY \|_F^2 \notag \\
+ \lambda_2 \sum_{c=1}^{C} \uuu_c^\top  \mathbf{D}_A^c \left( i + 1 \right) \uuu_c
\end{align}

Then, defining $\mathcal{J}_U \left( i \right) = \| \TT^\top \WW + \OO^\top \UU \left( i \right) - \YY \|_F^2$,
\begin{align}
\mathcal{J}_U \left( i + 1 \right) +
\lambda_2 \sum_{c=1}^C \uuu_c^\top \left( i + 1 \right) \mathbf{D}_A^c \left( i + 1 \right) \uuu_c \left( i + 1 \right) \notag \\
\leq \mathcal{J}_U \left( i \right) + \lambda_2 \sum_{c=1}^C \uuu_c^\top \left( i \right) \mathbf{D}_A^c \left( i \right) \uuu_c \left( i \right)
\end{align}

With the definition of $\mathbf{D}_A^c$,
\begin{align}
\label{eq:opt_less}
\mathcal{J}_U \left( i + 1 \right) +
\lambda_2 \sum_{c=1}^C \sum_{o=1}^{O} \sum_{m=1}^{M} \frac{\| \uuu_c^{o_m} \left( i + 1 \right) \|_2^2}{2 \| \uuu_c^{o_m} \left( i \right) \|_2} \notag \\
\leq \mathcal{J}_U \left( i \right) + \lambda_2 \sum_{c=1}^C \sum_{o=1}^{O} \sum_{m=1}^{M} \frac{\| \uuu_c^{o_m} \left( i \right) \|_2^2}{2 \| \uuu_c^{o_m} \left( i \right) \|_2}
\end{align}

According to Lemma \ref{lemma1}, for $c \in 1, \dots, C$ we can derive that
\begin{align}
\label{eq:norm_less}
\sum_{o=1}^O \sum_{m=1}^M \| \uuu_c^{o_m} \left( i + 1 \right) \|_2 -
\sum_{o=1}^O \sum_{m=1}^M \frac{\| \uuu_c^{o_m} \left( i + 1 \right) \|_2^2}{2 \| \uuu_c^{o_m} \left( i \right) \|_2} \notag \\
\leq \sum_{o=1}^O \sum_{m=1}^M \| \uuu_c^{o_m} \left( i \right) \|_2 -
\sum_{o=1}^O \sum_{m=1}^M \frac{\| \uuu_c^{o_m} \left( i \right) \|_2^2}{2 \| \uuu_c^{o_m} \left( i \right) \|_2}
\end{align}

After adding Eqs. (\ref{eq:opt_less}) and (\ref{eq:norm_less}), we obtain
\begin{align}\label{eq:Step6Proof}
\mathcal{J}_U \left( i + 1 \right) + \lambda_2 \sum_{c=1}^C \sum_{o=1}^O \sum_{m=1}^M \| \uuu_c^{o_m} \left( i + 1 \right) \|_2 \notag \\
\leq \mathcal{J}_U \left( i \right) + \lambda_2 \sum_{c=1}^C \sum_{o=1}^O \sum_{m=1}^M \| \uuu_c^{o_m} \left( i \right) \|_2
\end{align}

Therefore, Eq. (\ref{eq:Step5Proof}) and Eq. (\ref{eq:Step6Proof}) prove
that the value of the objective function decreases in every iteration when solving $\WW$ and $\UU$.
Because the objective function in Eq. (5) of the main paper is convex with respect to $\WW$ and $\UU$,
the proposed Algorithm 1 converges to the optimal solution. 
$\blacksquare$
